\documentclass[nohyperref]{article}

\usepackage{microtype}
\usepackage{graphicx}
\usepackage{subfigure}
\usepackage{booktabs} 
\usepackage{tabularx}
\usepackage{tabu}
\usepackage{hyperref}

\usepackage[accepted]{icml2022}

\usepackage{amsmath}
\usepackage{amssymb}
\usepackage{mathtools}
\usepackage{amsthm}

\usepackage[capitalize,noabbrev]{cleveref}

\theoremstyle{plain}
\newtheorem{theorem}{Theorem}[section]

\newtheorem{lemma}[theorem]{Lemma}

\theoremstyle{definition}

\theoremstyle{remark}

\newcommand{\norm}[1]{\left\lVert#1\right\rVert}
\usepackage[textsize=tiny]{todonotes}

\icmltitlerunning{SbF-Pruner: End-to-End Sensitivity-Based Filter Pruning}

\begin{document}

\twocolumn[
\icmltitle{End-to-End Sensitivity-Based Filter Pruning}



\icmlsetsymbol{equal}{*}

\begin{icmlauthorlist}
\icmlauthor{Zahra Babaiee}{xxx}
\icmlauthor{Lucas Liebenwein}{yyy}
\icmlauthor{Ramin Hasani}{yyy}
\icmlauthor{Daniela Rus}{yyy}
\icmlauthor{Radu Grosu}{xxx}
\end{icmlauthorlist}

\icmlaffiliation{xxx}{Computer Science Department, Technical University of Vienna, Vienna, Austria}
\icmlaffiliation{yyy}{Electrical Engineering and Computer Science Department, MIT, Boston, USA}

\icmlcorrespondingauthor{Zahra Babaiee}{zahra.babaiee@tuwien.ac.at}

\icmlkeywords{Machine Learning, ICML}

\vskip 0.3in
]




\begin{abstract}
In this paper, we present a novel sensitivity-based filter pruning algorithm (SbF-Pruner) to learn the importance scores of filters of each layer end-to-end. Our method learns the scores from the filter weights, enabling it to account for the correlations between the filters of each layer. Moreover, by training the pruning scores of all layers simultaneously our method can account for layer interdependencies, which is essential to find a performant sparse sub-network. Our proposed method can train and generate a pruned network from scratch in a straightforward, one-stage training process without requiring a pre-trained network. Ultimately, we do not need layer-specific hyperparameters and pre-defined layer budgets, since SbF-Pruner can implicitly determine the appropriate number of channels in each layer. Our experimental results on different network architectures suggest that SbF-Pruner outperforms advanced pruning methods. Notably, on CIFAR-10, without requiring a pretrained baseline network, we obtain 1.02\% and 1.19\% accuracy gain on ResNet56 and ResNet110, compared to the baseline reported for state-of-the-art pruning algorithms. This is while SbF-Pruner reduces parameter-count by 52.3\% (for ResNet56) and 54\% (for ResNet101), which is better than the state-of-the-art pruning algorithms with a high margin of 9.5\% and 6.6\%. 
\end{abstract}

\section{Introduction}

Convolutional Neural Networks (CNNs)~\cite{doi:10.1162/neco.1989.1.4.541} are widely used in various deep learning computer vision tasks. Large CNNs achieve considerable performance levels, but with significant computing, memory, and energy footprints \cite{sui2021chip}. These models are dense and over-parameterized. As a consequence, they cannot be effectively used in resource-limited environments such as mobile or embedded devices. Hence, it's crucial to create smaller models that can perform well without significantly sacrificing their accuracy and performance. This goal can be accomplished either through designing smaller network architectures \cite{lechner2020neural, pmlr-v97-tan19a} or through training an over-parameterized network and sparsifying it by pruning its redundant parameters \cite{han2016deep,liebenwein2020provable,liebenwein2021sparse}.

Neural network pruning is defined as systematically removing parameters from an existing neural network~\cite{hoefler2021sparsity}. It is a popular technique to reduce the growing energy and performance costs of neural networks and make it feasible to deploy them in resource-constrained environments such as smart devices. Various approaches have been developed to perform pruning as this has gained considerable attention over the past few years \cite{zhu2017prune, sui2021chip,liebenwein2021sparse,peste2021ac,frantar2021efficient,deng2020model}. 

\begin{figure}[t]
    \centering
    \vspace*{5ex}
    \includegraphics[scale=1.2]{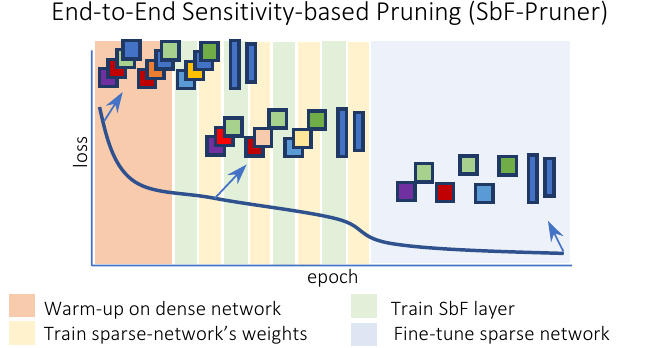}
    \caption{Sensitivity-based filter pruning schedule.}
    \label{fig:fig1}
\end{figure}
Pruning methods can be categorised into structured and unstructured. Unstructured methods reduce the size by removing individual weight parameters~\cite{han2016deep}, and structured methods remove parameters in groups, by pruning neurons, filters, or channels~\cite{10.1145/3005348, li2019learning, he2018soft, liebenwein2020provable}. Since our hardware is tuned for dense computations, structured pruning offers a favorable balance between accuracy and performance, providing more computational speedups~\cite{hoefler2021sparsity}.

Filter pruning is a prominent family of structured methods for CNNs. Choosing which filters to remove is the essential part of any filter pruning method. There are several different approaches for sorting the filters by a metric of importance in order to select the least important ones to remove. 

Various methods rely on the network structure and weights alone to determine the important filters rather than the training data. Magnitude pruning, for example, is one of the simplest and most common data-free methods, which prunes the filters with the smallest $l1$ norm of the weights. 
\begin{figure*}[t]
\centering
   \includegraphics[width=0.9\textwidth]{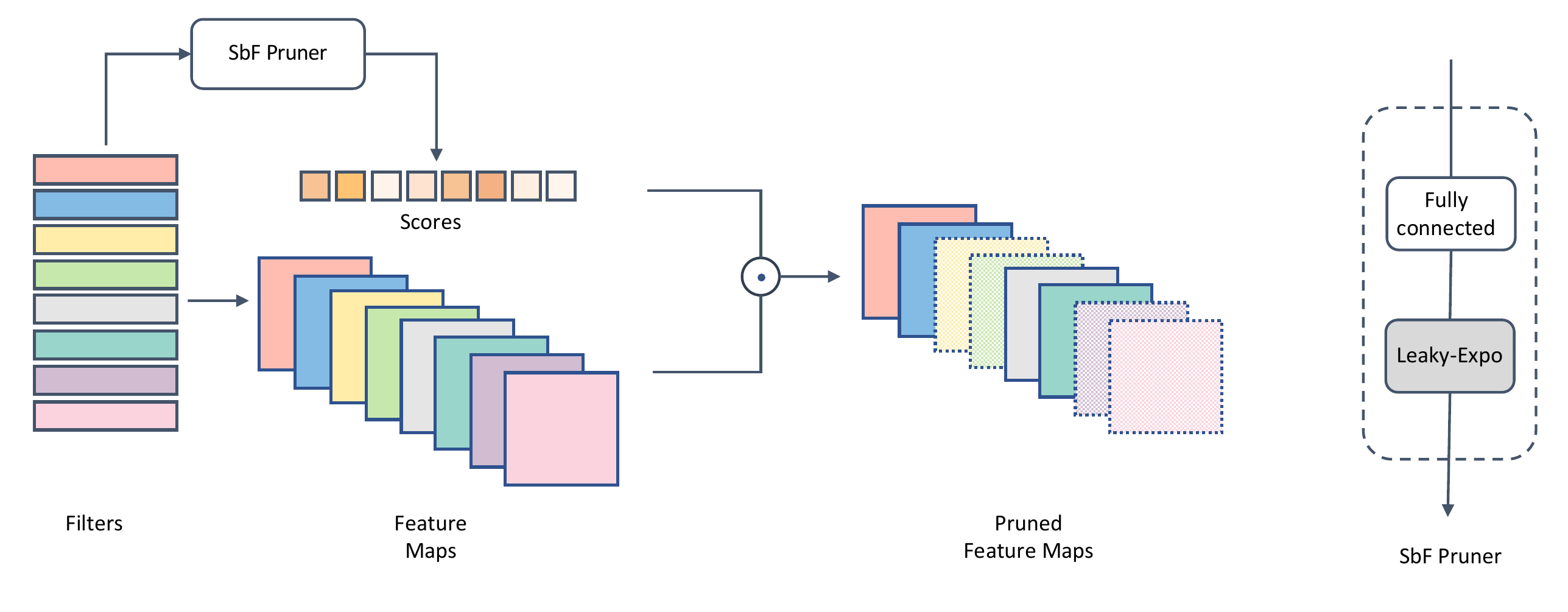}
   \vspace*{-1ex}
   \caption{The SbF-Pruner learns the importance scores of the filters from the filter weights.}
   \label{fig:att_pruner}
\end{figure*}
Many recent works introduce data-informed methods that focus on the feature maps generated from the training data (or a subset of samples) rather than the filters alone. These methods vary from sensitivity-based approaches, which consider the statistical sensitivity of the output feature maps to the input data~\cite{Malik:816737, liebenwein2020provable}, to correlation-based methods, with an inter-channel perspective, to keep the least similar or correlated feature maps~\cite{sun2015sparsifying, sui2021chip}.

Both data-free and data-informed methods often rely on finding the importance of filters of each layer locally. However, the importance of a filter changes relatively to the selection of the filters in previous and next layers. It is challenging to develop methods that can globally determine the importance of filters since CNNs should be considered as a whole. Moreover, a vital element of any pruning method is determining the optimal budget for each layer, which is a problem that all locally-determining importance-metric techniques face. The most trivial way to overcome these issues is to evaluate the network loss with and without each combination of $k$ candidate filters out of $N$. However, this approach requires the evaluation of ${N\choose k}$ subnetworks. It is clear that such computation is not practically feasible.

Training-aware pruning methods aim to learn binary masks to turn on and off each filter. A regularization metric often accompanies them to add a penalty for guiding the masks to the desired budget. Mask learning simultaneously for all filters is an effective method for identifying a globally optimal subset of the network. However, due to the discrete nature of the filters and binary values of the masks, the optimization problem is often non-convex and NP-hard. A simple trick used by many recent works~\cite{DMC2020,NPPM2021,DBLP:journals/corr/abs-2112-03406} is to use Straight-Through Estimators~\cite{STEBengio} in order to calculate the derivatives of the binary functions as they are identity functions. 

In this paper, we introduce a novel end-to-end pruning technique that simultaneously trains and prunes the network. Instead of learning a binary mask, we train continuous scores from the filter weights with a sensitivity-based filter-pruning (SbF-Pruner) mechanism. This way, we allow the importance scores to be learned via gradient descent to obtain a sparsified network. SbF-Pruner obtains Filter scores through a specialized activation function and automatically computes pruned feature maps as a result of the product of the filter sensitivity scores with feature maps of each layer.
Our SbF-Pruner pipeline is symbolically shown in Figure~\ref{fig:att_pruner}.

\paragraph{Our main contributions:}
\begin{itemize}
    \item We introduce the SbF-Pruner, \textit{an end-to-end algorithm that learns the importance scores directly from the network weights and filters}. Our method allows to extract hidden patterns in the filter weights for training the scores, rather than relying only on the weight magnitudes. The feature maps are multiplied by our learned scores during training. This way \textit{our method implicitly accounts for the data samples through loss propagation}, enabling our SbF-Pruner to enjoy the advantages of both data-free and data-informed methods.
    
    \item Our SbF-Pruner \textit{automatically calculates global importance scores for all filters and determines layer-specific budgets} with only one global hyper-parameter.
    
    \item We empirically investigate the pruning performance of our SbF-Pruner in various pruning tasks and compare it to advanced state-of-the-art pruning algorithms. \textit{Our method proves to be competitive and yield higher pruning ratios while preserving higher accuracy}.
\end{itemize}

\section{End-to-End Sensitivity-Based Filter Pruner}
In this section, we first denote our notation and procedurally construct our pruning algorithm.

\subsection{Notation}
To better describe our approach, we first introduce some useful notation. The filter weights of layer $l$ are represented by $\mathcal{F}_l\,{\in}\,{\rm I\!R}^{F \times C\times K \times K}$ where $F$ is the number of filters, $C$ the number of input channels, and $K$ the convolutional kernel size. The feature maps of layer $l$ are given by $A_l\in {\rm I\!R}^{F\times H \times W}$ where $H$ and $W$ are the image height and width respectively. For simplicity, we ignore the batch dimension.

\subsection{Learning the Scores}
The SbF-Pruner can be regarded as an independent layer, whose inputs are the layer filter weights and the outputs are the scores associated with each filter. The pruner layer is a linear transformation of the layer weights to a single vector whose length equals the number of filters of the layer, followed by an activation function which guarantees that the scores are between $0$ and $1\,{+}\,\epsilon$. We will explain in detail the choice of the activation function in the next section. 
\begin{equation}
\centering
\label{eq:scores}
S_{l} = \phi(\mathcal{F}_l\ W^{\mathcal{F}})
\end{equation}
Here $\phi$ is the sbf-activation function and the projection is a parameter matrix $W^{\mathcal{F}} \in {\rm I\!R}^{(F \times C\times K \times K) \times F}$.

The scores learned by the pruner layer are then pointwise (${\odot}$) multiplied by the feature maps of the same layer.

\begin{equation}
\centering
\label{eq:scores_mul}
A_{l}^{'} = S_{l}(\mathcal{F}_{l})\,{\odot}\,A_{l}
\end{equation}

While training the scores, we do not discretize them to binary values; instead, their continuous values are directly multiplied by the feature maps. The closer the filter score is to $1$, the more the corresponding feature map is preserved.

\subsection{The SbF-Pruner Activation Function}
SoftMax is the typical choise of activation function in additive attention for computing importance \cite{vaswani2017attention}. However, SoftMax is not a suitable choice for filter scores. While the range of its outputs is between $0$ and $1$, the sum of its outputs has to be $1$, meaning that all scores will have small values, or there would be only one score close to $1$.

In contrast to SoftMax, we would intuitively want that many filter scores could possibly be close to $1$. More formally, the scores should have the following three main attributes:
\begin{enumerate}
    \item All filter scores should have a positive value that ranges between $0$ and $1$, as is the case in SoftMax.
    \item All filter scores should adapt from their initial value of $1$, as we start with a completely dense model.
    \item The filter-scores activation function should have non-zero gradients over their entire input domain.
\end{enumerate}
Sigmoidal activations satisfy Attributes $1$ and $3$. However, they have difficulties with Attribute $2$. For high temperatures, sigmoids behave like steps, and scores quickly converge to $0$ or $1$. The chance these scores change again is very low, as the gradient is close to zero in these points.  Conversely, for low temperatures, scores have a hard time to converge to $0$ or $1$. Finding the optimal temperature needs extensive search for each of the layers separately. Moreover, starting from a dense network with all scores set to $1$ is not feasible with sigmoids.

To satisfy Attributes $1$-$3$, we designed our own activation function, as shown in Figure~\ref{fig:leaky_expo}. First, in order to ensure that scores are positive, we use an exponential activation function, and learn its logarithm values. Second, we allow the activation to be leaky, by not saturating it at $1$, as this would result in $0$ gradients, and scores getting stuck at $1$. Formally, our leaky-exponential activation function is defined as follows, where $a$ is a small value.:
\begin{equation}
\centering
\label{eq:activation_func}
\phi(x) = 
     \begin{cases}
       e^x &\quad\text{if $x<0$} \\ 
       1.0+ax &\quad\text{if $x\ge0$}
     \end{cases}
\end{equation}
%
%
\begin{figure}
\centering
   \includegraphics[width=0.3\textwidth]{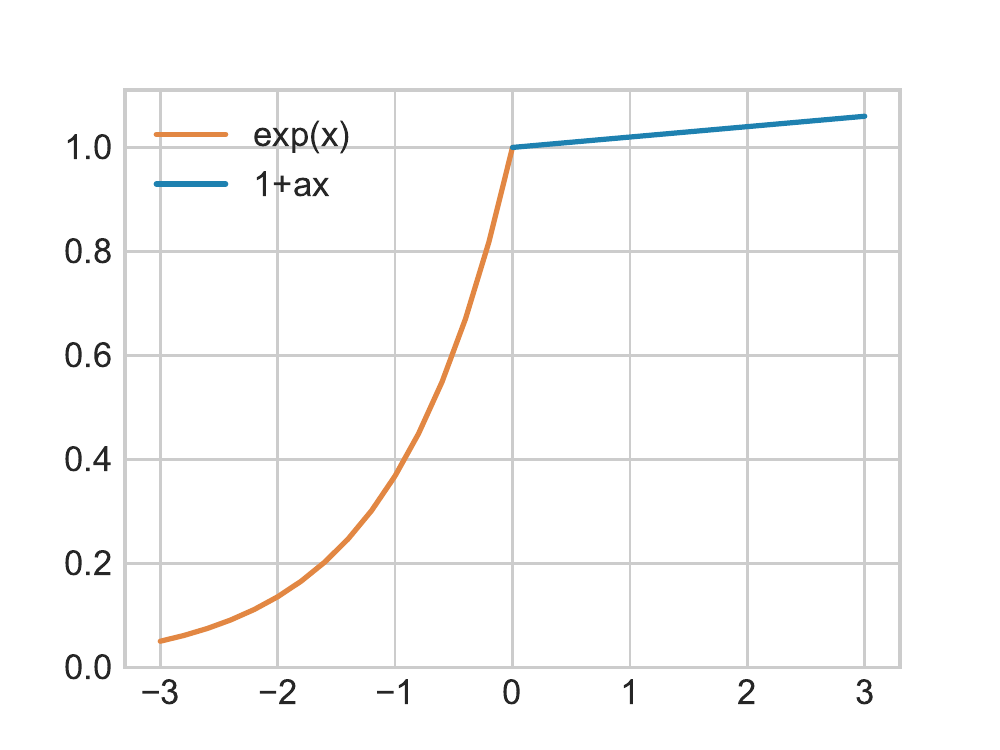}
   \caption{The leaky-exponential function used as the activation function for the SbF-Pruner layers.}
    \label{fig:leaky_expo}
    \vspace*{-4ex}
\end{figure}

\subsection{The Optimization Problem}
The network filter-pruning optimization problem can be formulated as in Equation~\eqref{eq:loss}, were $\mathcal{L}$ is the loss function, $f$ is the output function of the network with inputs $x$ and labels $y$, $\mathcal{W}$ are the network parameters, and $S$ are the pruner-layers parameters. By $\|g(S_l)\|_{1}$ we denote the $l_1$-norm of the binarized filter scores of layer $l$, and by $F_l$ the number of filters of layer $l$. Finally, $p$ is the pruning ratio. 

\begin{equation}
\centering
\label{eq:loss}
\begin{split}
\min_{S} \mathcal{L}(y, f(x; \mathcal{W}, S))\qquad\\
s.t.\ \ \sum_{l=0}^{L}\|g(S_l)\|_{1} - p\sum_{l=0}^{L}F_l = 0
\end{split}
\end{equation}

To address the pruning budget constraint, we add a regularization factor to the loss function while training the pruner layers. We add the $l_1$-norm of the continuous scores calculated for the filters of each layer to the classification loss:
\begin{equation}
\centering
\label{eq:regularized_loss}
\mathcal{L}^{'}(S) = \mathcal{L}(y, f(x; \mathcal{W}, S)) + \lambda \sum_{l=0}^{L}\|S_l\|_{1}
\end{equation}
Here $\lambda$ is a global hyper-parameter controlling the regularizer's effect on the total loss. With the loss function consisting of the classification and the $l_1$-norm of the scores, the pruning effect of each filter is reflected in the loss function. 

By incorporating the classification and $l_1$-norm of the scores in the loss function, the effect of the score of each filter is accounted for in the loss value in two ways. 
\begin{enumerate}
    \item When multiplied by small scores, feature maps have less influence on classification. This may result in a larger loss if these maps are important and vice versa.
    \item On the other hand, the part of the loss function consisting of the $l_1$-norm of the scores, decreases the loss value when there are more scores with small values.
\end{enumerate}
In some sense, this duality mimics the way synapses are pruned or strengthened through habituation or conditioning in nature. If the use of a synapse does not lead to a reward it is pruned, otherwise it is strengthened. Moreover, pruning is essential in order to avoid saturation and enable learning.

The value of the parameter $\lambda$ plays a very important role in keeping the balance between the two factors, and in controlling the pruning ratio. When $\lambda$ is set to $0$, all scores will stay at values close to $1$, and when it is set to a large value,  all scores are forced to eventually converge to $0$. 

During training of the SbF-Pruner layers, we freeze all other network parameters. The scores are directly learned from the filter weights, reflecting attributes such as the filter magnitude and the correlation among the different filters of the same layer. A globally pruned network is obtained by simultaneously training the scores of all layers. This allows to propagate the scored influence of the feature maps of one layer to the feature maps and scores of the next layers.

\subsection{The Training schedule}

We start training the network by a number of warm-up epochs. During the warm-up phase we train the dense network, with all filter scores fixed to $1$ and the pruner layer weights frozen. After the warm-up phase, we start training the pruner layers as well as the network weights.

We multiply the feature maps by the continues score values. With this approach, when a score is $1$, the corresponding feature map is fully preserved. As the score gets smaller, its feature map is getting weaker since it is multiplied by a value smaller than $1$, and finally almost completely pruned as the score gets closer to $0$. However, while the scores are getting trained, the network weights can learn to adapt to the feature map intensity. This means that, although a filter may have a low score, the loss of the feature map intensity can be compensated by magnifying the feature map itself. In order to prevent this, we use an alternating training approach \cite{peste2021ac}. 

We train the pruner layers for a few epochs, while all other network parameters are frozen. Then we alternate for training the sparse network weights and during these few epochs we freeze all pruner layer parameters. We multiply the feature maps to their learned scores so far, but we switch to discrete binary values of scores during the sparse network training phases. The binary values are obtained from the continuous scores with a gate function as defined bellow, where $s$ is the non-binary score value:
\begin{equation}
\centering
\label{eq:binary_scores}
g(s) = 
     \begin{cases}
       $1$ &\quad\text{if $s$}\ge0.5\\
       $0$ &\quad\text{otherwise.} \\ 
     \end{cases}
\end{equation}
After the pruner training phase is completed, we complete the training procedure with fine-tuning the sparse model. 

Our proposed method does not necessarily need a pre-trained dense model to find the importance of filters or feature maps. Even when starting from the scratch and after a warm-up phase, it can find an optimal sparse sub-network by training the weights and pruner-layers simultaneously.

\subsection{The Algorithm}
In previous secions, we introduced the basic concepts of the SbF-Pruner. Algorithm~\ref{alg:1} outlines the SbF pruning algorithm for the scores training stage. When training the parameters of the SbF-Pruner layers, Equation~\eqref{eq:activation_func} is used as the activation function, and during training the rest of the network parameters, we use Equation~\eqref{eq:binary_scores}.
\begin{algorithm}[t]
\caption{Sensitivity-Based Filter Pruning}
\label{alg:1}
\begin{algorithmic}
\STATE \textbf{Inputs:} mini batches $\mathcal{B}$, Network $N$ parametrized by $\mathcal{W}$ trained for number of warm-up epochs, SBF-Pruner layers parametrized by $S$, Regularization hyper-parameter $\lambda$, Pruning-training epochs $E_p$, Scores training phase mini-epochs $E_s$, Fine-tuning epochs $E_f$.
\STATE \textbf{Output:} Pruned Network with Scores for Fine-tuning
\FOR{$j$ in $E_p$}
\FOR{$j$ in $E_s$}
\FOR{$b$ \textbf{in} $\mathcal{B}$}
\STATE forward calculation:\\
\quad $\min_{S} \mathcal{L}(y, f(x; \mathcal{W}, S)) + \lambda \sum_{l=0}^{L}\|S_l\|_{1}$
\STATE calculate gradients w.r.t $S$.
\STATE Update $S$ by optimizer.
\ENDFOR
\ENDFOR
\FOR{$j$ in $E_s$}
\FOR{$b$ \textbf{in} $\mathcal{B}$}
\STATE forward calculation:\\
\quad $\min_{\mathcal{W}} \mathcal{L}(y, f(x; \mathcal{W}, S))$
\STATE calculate gradients w.r.t $\mathcal{W}$.
\STATE Update $\mathcal{W}$ by optimizer.
\ENDFOR
\ENDFOR
\ENDFOR

\STATE \textbf{Return} $N_\mathcal{W}$
\end{algorithmic}
\end{algorithm}

\subsection{Prunner's Effect on Iterative Inference in ResNets}
In this section, we proceed with an analysis of the effect that the SbF-Pruner has on the iterative feature-refinement process in residual networks \cite{greff2016highway}. 

ResNets \cite{he2016identity} are known to enhance representation learning in deeper layers via an iterative feature-refinement scheme \cite{greff2016highway}. This scheme suggests that input features to a ResNet do not create new representations. Rather, they gradually and iteratively refine the learned features of the initial residual blocks \cite{jastrzebski2018residual}. Iterative refinement of features has been shown to be necessary for obtaining attractive levels of performance with ResNet, while their disruption hurts performance. 

As the SbF-Pruner layer modifies the underlying model structure and the feature maps of a residual block, it is very important to investigate if the SbF-Pruner preserves the iterative feature-refinement property of ResNets.

To make this analysis precise, let us first formalize iterative inference as discussed in~\cite{jastrzebski2018residual}: A residual block $i$ in a ResNet with $M$ blocks, performs for the input feature $\textbf{x}_i$ the following transformation: $\textbf{x}_{i+1} = \textbf{x}_i\,{+}\,f_i(\textbf{x}_i)$. Hence, the following loss function $\mathcal{L}$ can be recursively applied to the network \cite{jastrzebski2018residual}:
\begin{equation}
     \mathcal{L}(\textbf{x}_{M}) = \mathcal{L}(\textbf{x}_{M-1} + f_{M-1}(\textbf{x}_{M-1})).
\end{equation}
A first-order Taylor expansion of this loss, while ensuring that $f_{j}$'s magnitude is small, is a good approximation to formally investigating the iterative inference \cite{jastrzebski2018residual}. We thus obtain the following:
\begin{equation}
    \mathcal{L}(\textbf{x}_{M}) = \mathcal{L}(\textbf{x}_{i}) + \sum^{M-1}_{j=i} f_{j}(\textbf{x}_{j}). \frac{\partial \mathcal{L}(\textbf{x}_{j})}{\textbf{x}_{j}} + \mathcal{O}(f^2_{j}(\textbf{x}_{j})).
    \label{eq:loss_extention}
\end{equation}
This approximation reveals that the $i$-th residual block, modifies features $\textbf{x}_{i}$ with roughly the same amount of $f_{i}(\textbf{x}_{i})$ as that of $\frac{\partial \mathcal{L}(\textbf{x}_{i})}{\textbf{x}_{i}}$. This implies a moderate reduction of loss as we transition from the $i$-th to  the $M$-th block, as an iterative refinement scheme~\cite{greff2016highway,jastrzebski2018residual}. The refinement step for a vanilla residual block can be computed by the squared norm of $f_{i}(\textbf{x}_{i})$, and can be normalized to the input feature as: $\norm{f_{i}(\textbf{x}_{i})}^{2}_{2}{/}\norm{\textbf{x}_{i}}^{2}_{2}$. 

Any modification to the structure of the residual network (e.g., if we use the SbF-Pruner) causes a change in the refinement step. This step has to be investigated if it does or it does not modify the iterative inference scheme.

\begin{lemma} 
\label{lem:1}
The iterative feature-refinement scale is bounded for ResNets with SbF-Pruner as follows, with parameter $\epsilon$ from the leaky integrator and $0\,{<}\,\delta\,{\leq}\,1\,{+}\,\epsilon$:
\begin{equation}
\label{eq:bounds}
\hspace*{-2ex}\frac{\delta}{1+\epsilon} \frac{\norm{f_{i}(\textbf{x}_{i})}^{2}_{2}}{\norm{ \textbf{x}_{i}}^{2}_{2}}\,{\leq}\,\frac{\norm{S_i\,{\odot}\,f_{i}(\textbf{x}_{i})}^{2}_{2}}{\norm{S_i \odot \textbf{x}_{i}}^{2}_{2}}\,{\leq}\,\frac{1+\epsilon}{\delta} \frac{\norm{f_{i}(\textbf{x}_{i})}^{2}_{2}}{\norm{ \textbf{x}_{i}}^{2}_{2}} 
\end{equation}
\end{lemma}

\begin{proof} A ResNet block $i$ that is equipped with an SbF-Pruner layer, transforms the features $\textbf{x}_{i}$ with the following expression: $\textbf{x}_{i+1}\,{=}\,S_i\,{\odot}\,(\textbf{x}_i\,{+}\,f_i(\textbf{x}_i)$), where $S_i$ stands for the score vector computed by the SbF-Pruner. 
The refinement step is given by $\norm{S_i\,{\odot}\,f_{i}(\textbf{x}_{i})}^{2}_{2}$ and its input-normalized representation is $\norm{ S_i\,{\odot}\,f_{i}(\textbf{x}_{i})}^{2}_{2}/\norm{S_i\,{\odot}\, \textbf{x}_{i}}^{2}_{2}$.

\textit{Deriving the upper bound:} Assuming that every element in $S_i$ is between $\delta$ and $1+\epsilon$, for $0\,{<}\,\delta\,{\leq}\,1\,{+}\,\epsilon$, we have:
\begin{equation}
\norm{S_i\,{\odot}\,f_i(\textbf{x}_{i})}\,{\leq}\,\norm{S_i}_\infty \norm{ f_i(\textbf{x}_{i})}\,{\leq}\, (1+\epsilon) \norm{ f_i(\textbf{x}_{i})}
\end{equation}
\vspace*{-3ex}
\begin{equation}
\norm{S_i\,{\odot}\,\textbf{x}_{i}}\,{\geq}\,\norm{S_i}_{min} \norm{\textbf{x}_{i}}\,{\geq}\, \delta \norm{ \textbf{x}_{i}}. \end{equation}

As a consequence, the following upper-bound inequality holds for the iterative inference: 
\begin{equation}
\label{eq:itpU}
\frac{\norm{S_i\,{\odot}\,f_{i}(\textbf{x}_{i})}^{2}_{2}}{\norm{S_i \odot\textbf{x}_{i}}^{2}_{2}} \leq \frac{1+\epsilon}{\delta} \frac{\norm{f_{i}(\textbf{x}_{i})}^{2}_{2}}{\norm{ \textbf{x}_{i}}^{2}_{2}} 
\end{equation}

\textit{Deriving the lower bound:} Assuming that every element in $S_i$ is between $\delta$ and $1+\epsilon$, we have:
\begin{equation}
\norm{S_i\,{\odot}\,f_i(\textbf{x}_{i})}\,{\geq}\,\delta \norm{ f_i(\textbf{x}_{i})}
\end{equation}
\vspace*{-3ex}
\begin{equation}
\norm{S_i\,{\odot}\,\textbf{x}_{i}}\,{\leq}\,(1+\epsilon) \norm{ \textbf{x}_{i}}. 
\end{equation}

As a consequence, the following lower-bound inequality holds for the iterative inference: 
\begin{equation}
\label{eq:itpL}
\frac{\norm{S_i\,{\odot}\,f_{i}(\textbf{x}_{i})}^{2}_{2}}{\norm{S_i \odot\textbf{x}_{i}}^{2}_{2}} \geq \frac{\delta}{1+\epsilon} \frac{\norm{f_{i}(\textbf{x}_{i})}^{2}_{2}}{\norm{ \textbf{x}_{i}}^{2}_{2}} 
\end{equation}

Inequalities~\eqref{eq:itpU} and~\eqref{eq:itpL} prove the the stated lemma.
\end{proof}

Lemma \ref{lem:1} has profound implications in practice. It indicates that the iterative-inference property of the ResNets equipped with an SbF-Pruner is both lower and upper bounded. These ResNets not only get compressed in size, but also preserve the representation learning capabilities of ResNets between these two bounds. The bounds themselves can be fine tuned with the parameter $\lambda$ of Equation~\eqref{eq:regularized_loss}.

\section{Experiments}
In this section, we first present our implementation details and then discuss our experimental results.

\textbf{Baselines.} We compare our method to numerous standard and advanced pruning methods. The models include: L1 Norm \cite{li2017pruning}, Neuron Importance Score Propagation (NISP) \cite{nisp2018}, Soft Filter Pruning (SFP)~\cite{SFP2018}, Discrimination-aware Channel Pruning (DCP)~\cite{DCP2018}, DCP-Adapt~\cite{DCP2018},  Collaborative  Channel  Pruning (CCP)~\cite{CCP2019}, Generative Adversarial Learning (GAL)~\cite{GAL2019}, Filter Pruning using High-rank feature map (HRank)~\cite{HRank2020}, Discrete model compression (DMC)~\cite{DMC2020},  Network pruning via performance maximization (NPPM)~\cite{NPPM2021}, Channel independence-based pruning (CHIP)~\cite{sui2021chip}, and Auto-ML for Model Compression (AMC)~\cite{AMC2018}.

\subsection{Implementation Details}

\textbf{Training Procedure.} We conducted experiments on CIFAR-10 dataset with two different models, ResNet56 and ResNet110. For both models, we do not use a fully trained network as the baseline to prune. We train each model for 50 warm-up epochs. During the warmup (See Figure \ref{fig:fig1}), we use batch size of $256$ and Stochastic Gradient Descent (SGD) as optimizer, with 0.1 as initial learning rate, 0.9 as momentum, and 0.0005 for the weight decay. 

After the warm-up stage, we start training all layers of the SbF-Pruner, alternatively, with the rest of network parameters. We train the scores with the SbF-Pruner layers for 3 epochs, while all other parameters are frozen. We add the regularized loss defined in Equation~\eqref{eq:loss_extention} when we train the scores. Then we switch to training the network weights for 6 epochs, while this time the SbF-Pruner layers are frozen and we use the classification loss. In this phase, we use Equation~\eqref{eq:binary_scores} as the gate function for the filter scores. When the SbF-Pruner layers are training, the feature maps get multiplied by the continuous value of the scores. We repeat these two phases for 10 times, summing up to training for 90 epochs. We use the ADAM optimizer with learning rates $10^-6$ and $10^-3$ for training the SbF-Pruners and the network parameters, respectively.

\begin{table*}[t]
\label{tb:results}
  \centering
  \caption{Experiment results on Cifar10 dataset with ResNet56 and ResNet110.\\ $\Delta$ shows the difference in the accuracy of the base dense model used and that of the pruned network. SbF-Pruner does not use a pre-trained base model. Numbers are taken from the reported results in the cited papers.}
  \vspace*{1.5ex}
  \begin{tabular}{lcccccc}
    \toprule
    \multicolumn{1}{c}{} & \multicolumn{5}{c}{ResNet56}                   \\
   \midrule
    \textbf{Method}     & Baseline Acc($\%$)    & Pruned Acc($\%$)   & $\Delta$($\%$)    & $\downarrow$Parameters($\%$)    & $\downarrow$Flops($\%$)  \\
    \midrule
    $l1$ Norm~\cite{li2017pruning} & 93.04  & 93.06  & +0.02 & 13.7  &  27.6 \\
    NISP~\cite{nisp2018} & N/A & N/A  & -0.03 &  42.2 &  35.5    \\
    SFP~\cite{SFP2018} & 93.59  & 93.78  & +0.19 &  N/A &  41.1    \\
    DCP~\cite{DCP2018} & 93.80  & 93.59  & -0.31 & N/A  &  50.0    \\
    DCP-Adapt~\cite{DCP2018} & 93.80  & 93.81  & +0.01 & N/A  &  47.0    \\
    CCP~\cite{CCP2019} & 93.50  & 93.46  & -0.04 & N/A  &  47.0    \\
    GAL~\cite{GAL2019} & 93.26 & 93.38  & +0.12 & 11.8  &  37.6    \\
    HRank~\cite{HRank2020} & 93.26  & 93.52  & +0.26 & 16.8  &  29.3    \\
    DMC~\cite{DMC2020} & 93.62  & 93.69  & +0.07 & N/A  &   50.0   \\
    NPPM~\cite{NPPM2021} & 93.04  & 93.40  & +0.36 &  N/A &  50.0    \\
    CHIP~\cite{sui2021chip}  & 93.26 & 94.01  & +0.75 & 42.8 & 47.4  \\
    \textbf{Ours}    &  & \textbf{94.28} &  & \textbf{52.3} & 49.3 \\
    \midrule
    AMC~\cite{AMC2018} &  92.80 & 91.90  & -0.90 & N/A  &  50.0    \\
    GAL~\cite{GAL2019} & 93.26 & 91.58  & -1.68 & 65.9  & 60.2    \\
    HRank~\cite{HRank2020} & 93.26  & 90.72  & -2.54 & 68.1  &  74.1    \\
    CHIP~\cite{sui2021chip}  & 93.26 & 92.05  & -1.21 & 71.8 & 72.3  \\
    \textbf{Ours}    &  & \textbf{92.42} &  & \textbf{83.0} & 70.5 \\
    \midrule
    \multicolumn{1}{c}{} & \multicolumn{5}{c}{ResNet110}                   \\
   \midrule
    $l1$ Norm~\cite{li2017pruning} & 93.53  & 93.30  & -0.23 & 32.4  & 38.7 \\
    NISP~\cite{nisp2018} & N/A & N/A  & -0.18 &  43.78 &  43.25    \\
    SFP~\cite{SFP2018} & 93.68  & 93.86  & +0.18 &  N/A &  40.8    \\
    GAL~\cite{GAL2019} & 93.50 & 93.59  & +0.09 & 18.7  &  4.1    \\
    HRank~\cite{HRank2020} & 93.50  & 94.23  & +0.73 & 39.4  &  41.2    \\
    CHIP~\cite{sui2021chip} & 93.50 & 94.44  & +0.94 & 48.3 & 52.1  \\
    \textbf{Ours}    &  & \textbf{94.69} &  & \textbf{54.9} & 51.3 \\
    \midrule
    GAL~\cite{GAL2019} & 93.50 & 92.74  & -0.76 & 44.8  &  48.5    \\
    HRank~\cite{HRank2020} & 93.50  & 92.65  & -0.85 & 68.7  &  68.6    \\
    CHIP~\cite{sui2021chip} & 93.50 & 93.63  & +0.13 & 68.3 & 71.3  \\
    \textbf{Ours}    &  & \textbf{93.68} &  & \textbf{79.3} & \textbf{74.8} \\
    \bottomrule
  \end{tabular}
  \label{light0-table}
  \vspace*{-2ex}
\end{table*}

After the scores-training stage is finished, we start fine-tuning the models. In this stage, we remove all SbF-Pruner layers from the network, keeping the final score vectors trained for each layer. We use Equation~\eqref{eq:binary_scores} to select the binary values of scores, which is our ultimate goal when pruning. Each feature map either is removed entirely (having $0$ score), or is completely preserved (having $1$ score). We use SGD optimizer with the same parameters as the warm-up stage, and tune the networks for 300 epochs. 

\textbf{Balancing the Pruned Parameters and Flops.} When training the SbF-Pruner, we use the regularized loss to guide the scores to our desired budget. However, the number of parameters of each layer is different from the number of flops required for that layer:
\begin{equation}
\centering
\label{eq:params}
\begin{split}
(Params)_l = C_l\times F_l\times K_l \times K_l\\
(Flops)_l = C_l\times F_l\times K_l \times K_l \times W_l \times H_l
\end{split}
\end{equation}
Since the number of flops is also dependent on the image size, early convolutional layers before the max-pooling layers require more flops as the image sizes are still large. In order to keep the balance between pruned parameters and flops, we multiply the $l_1$ norm of each layer by the relative input image sizes of that layer to the last layers.

For experiments on ResNet56, we used $\lambda=5\times 10^{-4}$ and $\lambda=15\times 10^{-4}$ for pruning ratios of $52.3\%$ and $83.0\%$ respectively. On ResNet110, we set $\lambda$ to $2\times 10^{-4}$ and $5.5\times 10^{-4}$ to prune $54.9\%$ and $79.3\%$ of the network parameters.

\subsection{Performance}
\textbf{Experimental results.}
Table \ref{tb:results} shows our experimental results. We compare the performance of the SbF-Pruner with the sate-of-the-art filter pruning methods. As one can see, the SbF-Pruner outperforms other pruning methods, achieving higher accuracy while pruning more parameters. Specifically, on ResNet56 the SbF-Pruner results in an increase in accuracy even compared to the dense baselines while pruning $52.3\%$ of parameters and $49.3\%$ of flops. 

The same is true for ResNet110, where the SbF-Pruner achieves an accuracy of $94.69\%$ while it prunes $54.9\%$ of the parameters and $51.3\%$ of the flops. For higher pruning ratios, the SbF-Pruner outperforms CHIP~\cite{sui2021chip}, the next best method, by $0.37\%$ accuracy increase while pruning $12.8\%$ more parameters, with a total of $83.0\%$ of pruned parameters on ResNet56. Similarly, on ResNet110, the SbF-Pruner achieves higher accuracy than CHIP with $11\%$ more parameters pruned.

\textbf{Per-layer Budget Discovery.} The SbF-Pruner finds the optimal sparse sub-networks in a fully-automated pipeline and does not require budget allocation schedules per layer. Figure~\ref{fig:budgets} shows the discovered networks in our experiments from ResNet56 and ResNet110.

\begin{figure*}[t]
\label{fig:budgets}
\centering
   \includegraphics[width=0.99\textwidth]{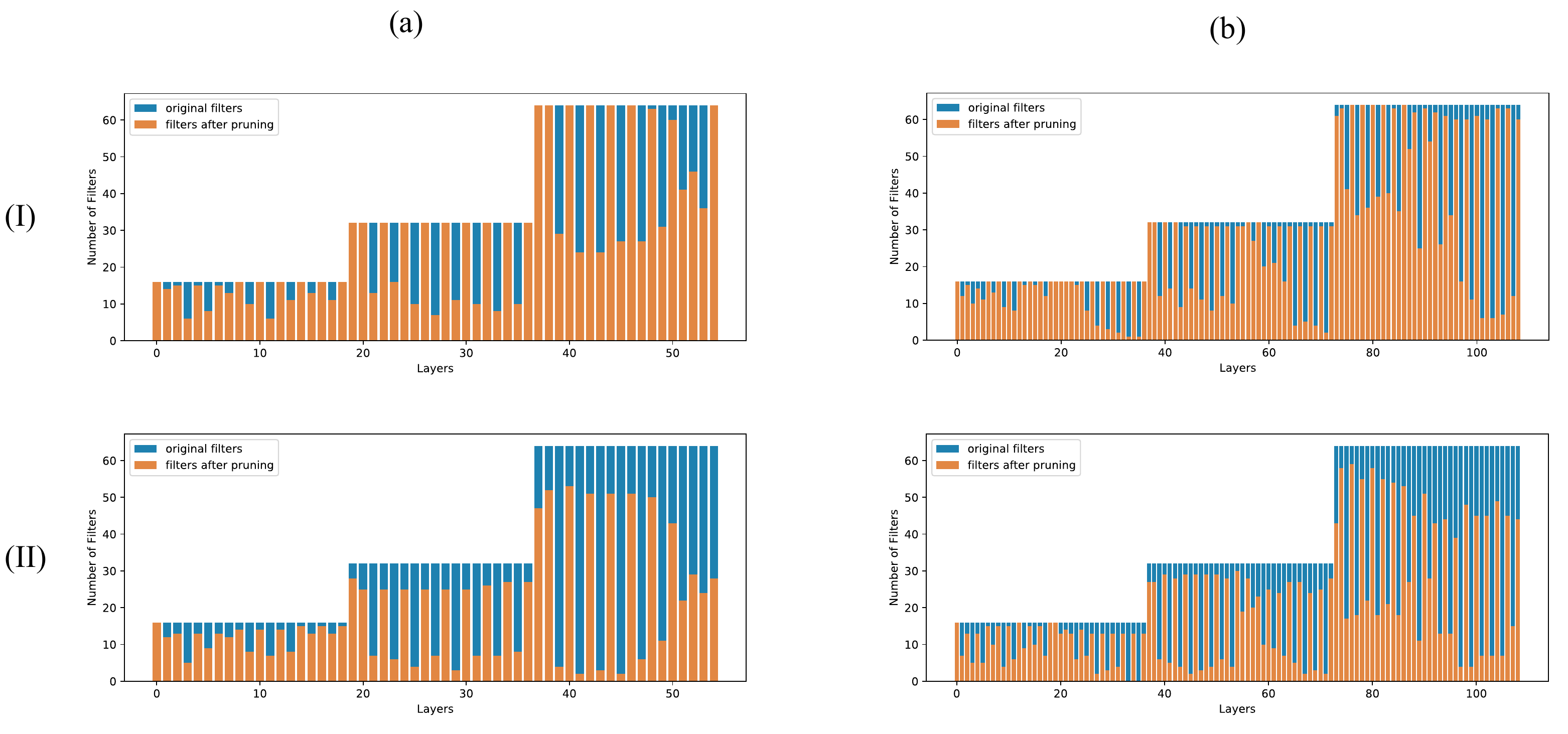}
   \caption{Networks discovered by the SbF-Pruner from a) ResNet56 and b) ResNet110 with medium (I) and (II) high pruning ratios. The SbF-Pruner automatically discovers the optimal per-layer budgets and does not require any pre-defined budgets.}
    \label{fig:att_pruner_bottlenecks}
   
\end{figure*}

In each residual block of the sub-networks learned by the SbF-Pruner, the first layer has lower number of filters remaining after pruning, compared to the second layer. This structure is similar to the bottle-neck architecture present in ResNets with large number of layers. It enables the network to concentrate on the most important features with less capacity, which is exactly what we are looking for with pruning. Many existing pruning algorithms use similar bottle-neck structures, when manually defining layer budgets for pruning~\cite{sui2021chip, HRank2020}. The SbF-Pruner is able to discover this pattern automatically, without supervision. Moreover, for high pruning ratios on ResNet110, there are blocks emerging with $0$ remaining filters in the first layers. This shows that the SbF-Pruner can also remove entire layers when required, having more freedom in the possible sparse sub-networks search space.

\section{Related Work}
Filter pruning is a very popular structured method for sparsifying CNNs, which supports storage reduction and processing efficiency without requiring any special library. One can roughly classify existing filter-pruning methods into three main categories based on their selection approaches: Data-free, data-informed, and training-aware methods.

\textbf{Data-Free Filter Pruning.} 
Following-up on the weights-magnitude-pruning method, where the weights with the smallest absolute values are considered the least important,~\cite{li2017pruning} use the sum of absolute values of the weights in a filter (the $l1$ norm) to prune the filters with the smallest weight values. ~\cite{SFP2018} dynamically prune the filters with the smallest $l2$ norm value in each epoch by setting them to zero and repeat this in each epoch during training. ~\cite{he2019filterFPGM} use the geometric median of filters as the pruning criterion. Although data-free methods can gain acceptable performance levels, several later works have shown that considering the training data will notably improve pruning precision~\cite{hoefler2021sparsity}. 

\textbf{Data-Informed Filter Pruning.}
Many pruning methods focus on the feature maps since they provide rich information from the data distribution as well as the filters. ~\cite{HRank2020} prune the filters whose feature maps have the lowest ranks, and ~\cite{liebenwein2020provable} use a sensitivity measure to prune filters with lowest effect on the outputs, providing provable sparsity guaranties. Motivated by the importance of inter-channel perspective for pruning,~\cite{sui2021chip} use the nuclear norm of the feature maps as an independence metric to prune the filters whose feature maps are the most dependent on the others.

\textbf{Training-Aware Filter Pruning.}
These methods use the power of training to learn a filter importance metric or guide the network to a sparse structure. Based on the idea of magnitude pruning, some methods add regularization factors to the loss to directly guide filters to close to zero values. ~\cite{NIPS2016_41bfd20a} and ~\cite{louizos2018learning} use Group Lasso and $l0$ regularization, respectively. Instead of solely relying on the weight magnitudes,~\cite{DCP2018} proposes a discrimination-aware channel-pruning method by defining per-layer classification losses. ~\cite{DMC2020} train binary gate functions with Straight Through Estimators and ~\cite{NPPM2021} focus on training binary gates by directly maximizing the accuracy of subnetworks.

No method discussed above however, is able to directly learn the importance-scores from the filter-weights, extract hidden correlations among filters, automatically calculate global importance scores for all filters and determine layer-specific budgets, all at the same time during training, thus taking advantage of both data-free and data-informed methods.

\section{Conclusion}
We proposed the SbF-Pruner an end-to-end sensitivity-based filter-pruning algorithm that learns importance scores via gradient descent. In contrast to a large spectrum of advanced pruning algooorithms, ours does not require a baseline pretrained network to prune from. It rather sparsifies dense networks from scratch, through cycles of gradient descent. 

We showed comprehensively that much better compression rates are achievable through the use of the SbF-Pruner for residual networks, while maintaining a competitively high accuracy. The SbF-Pruner computes global importance scores for filters and automatically associates pruning budgets to a neural network's layers with a single hyperparameter. This way, the algorithm takes advantage of both the data-free and the data-informed methods. 

We hope that future work could begin using our proposed SbF-Pruner framework in resource-hungry applications domains, such as neural architecture search \cite{mellor2021neural}, and obtain compressed neural models endowed with salient features, automatically distilled during training, for resource-constraint environments. 

\newpage
\bibliography{references}
\bibliographystyle{icml2022}



\end{document}